%% file: main.tex
\documentclass{article}

\usepackage{microtype}
\usepackage{graphicx}
\usepackage{subfigure}
\usepackage{booktabs} % for professional tables
\usepackage{subcaption}
\usepackage{tabularx}
\usepackage{array}
\newcolumntype{Y}{>{\centering\arraybackslash}X}
\usepackage[dvipsnames]{xcolor}

\usepackage{hyperref}

\usepackage[accepted]{icml2025}

\usepackage{amsmath}
\usepackage{amssymb}
\usepackage{mathtools}
\usepackage{amsthm}

\usepackage[capitalize,noabbrev]{cleveref}

\theoremstyle{plain}
\newtheorem{theorem}{Theorem}[section]

\theoremstyle{definition}

\theoremstyle{remark}

\usepackage[textsize=tiny]{todonotes}

\icmltitlerunning{FPBoost: Fully Parametric Gradient Boosting for Survival Analysis}

\begin{document}

\twocolumn[
\icmltitle{FPBoost: Fully Parametric Gradient Boosting for Survival Analysis}

% It is OKAY to include author information, even for blind
% submissions: the style file will automatically remove it for you
% unless you've provided the [accepted] option to the icml2025
% package.

% List of affiliations: The first argument should be a (short)
% identifier you will use later to specify author affiliations
% Academic affiliations should list Department, University, City, Region, Country
% Industry affiliations should list Company, City, Region, Country

% You can specify symbols, otherwise they are numbered in order.
% Ideally, you should not use this facility. Affiliations will be numbered
% in order of appearance and this is the preferred way.
\icmlsetsymbol{equal}{*}

\begin{icmlauthorlist}
\icmlauthor{Alberto Archetti}{polimi}
\icmlauthor{Eugenio Lomurno}{polimi}
\icmlauthor{Diego Piccinotti}{chattermill}
\icmlauthor{Matteo Matteucci}{polimi}
\end{icmlauthorlist}

\icmlaffiliation{polimi}{Department of Electronics, Information, and Bioengineering, Politecnico di Milano, 20133 Milan, Italy}
\icmlaffiliation{chattermill}{Chattermill}

\icmlcorrespondingauthor{Alberto Archetti}{alberto.archetti@polimi.it}

% You may provide any keywords that you
% find helpful for describing your paper; these are used to populate
% the "keywords" metadata in the PDF but will not be shown in the document
\icmlkeywords{Machine Learning, Survival Analysis, Gradient Boosting, XGBoost, FPBoost}

\vskip 0.3in
]

% this must go after the closing bracket ] following \twocolumn[ ...

% This command actually creates the footnote in the first column
% listing the affiliations and the copyright notice.
% The command takes one argument, which is text to display at the start of the footnote.
% The \icmlEqualContribution command is standard text for equal contribution.
% Remove it (just {}) if you do not need this facility.

\printAffiliationsAndNotice % leave blank if no need to mention equal contribution
% \printAffiliationsAndNotice{\icmlEqualContribution} % otherwise use the standard text.

\begin{abstract}
Survival analysis is a statistical framework for modeling time-to-event data. It plays a pivotal role in medicine, reliability engineering, and social science research, where understanding event dynamics even with few data samples is critical. Recent advancements in machine learning, particularly those employing neural networks and decision trees, have introduced sophisticated algorithms for survival modeling. However, many of these methods rely on restrictive assumptions about the underlying event-time distribution, such as proportional hazard, time discretization, or accelerated failure time. In this study, we propose FPBoost, a survival model that combines a weighted sum of fully parametric hazard functions with gradient boosting. Distribution parameters are estimated with decision trees trained by maximizing the full survival likelihood. We show how FPBoost is a universal approximator of hazard functions, offering full event-time modeling flexibility while maintaining interpretability through the use of well-established parametric distributions. We evaluate concordance and calibration of FPBoost across multiple benchmark datasets, showcasing its robustness and versatility as a new tool for survival estimation.
\end{abstract}

\section{Introduction}
\label{sec:introduction}

Survival analysis is a field of statistics that plays a central role in data analysis for healthcare, providing the ability to estimate the timing and associated uncertainty of clinical events. This capability is essential to help physicians make informed safety-critical decisions based on data. Beyond healthcare, survival analysis has found applications in various fields, such as predicting equipment failures in industry or forecasting customer churn in relationship management. This widespread adoption underscores the importance of temporal risk estimation in various real-world scenarios~\cite{wang2019machine}.

The primary objective of survival models is to construct a time-dependent function $S(t|\mathbf{x})$ conditioned on a set of features $\mathbf{x}$, such as clinical indicators for hospitalized patients, known as the survival function. This function represents the probability that an event of interest has not occurred by time $t$, expressed as
\begin{equation*}
S(t | \mathbf{x}) = P(T > t | \mathbf{x}).
\end{equation*}
In practical applications, the event of interest can take several forms. In the healthcare context, for example, it can denote patient mortality, disease recurrence, or hospital discharge. As another example, in customer relationship management it might represent a client's initial purchase~\cite{klein2003survival}.

In order to model the survival function, traditional methods often rely on simplifying assumptions, such as the risk proportion between different subjects being constant over time~\cite{cox1972regression}. These assumptions allow the construction of survival functions from a small set of parameters that can be estimated with statistical methods. While suited for contexts with limited data, these simplifying assumptions constrain generalization in real-world scenarios~\cite{katzman2018deepsurv}. Machine learning techniques have allowed to advance these models by incorporating decision trees and neural networks, significantly enhancing their ability to identify and learn non-linear interactions within high-dimensional features. However, most of these techniques still operate under certain constraints, such as time discretization in neural-based approaches~\cite{kvamme2021continuous} or accelerated failure time in tree-based gradient boosting~\cite{collett2023modelling}.

In this context, we introduce Fully Parametric Gradient Boosting (FPBoost), a novel architecture designed to model hazard functions through the composition of multiple fully parametric hazard functions. Hazard functions are related to survival functions as they measure the instantaneous risk of a subject experiencing the event of interest. FPBoost combines the strengths of tree-based ensemble learning with gradient boosting~\cite{friedman2001greedy,collett2023modelling}, offering a flexible model with robust generalization capabilities and minimal assumptions. Modeling hazard as a weighted sum of multiple, fully-parametric functions, referred to as heads, allows FPBoost to be trained by maximizing the full survival likelihood~\cite{wang2019machine}. This, in turn, removes the need for simplified assumptions such as partial likelihood~\cite{cox1972regression,katzman2018deepsurv} or discrete losses~\cite{kvamme2021continuous}. Additionally, the continuous nature of the learned survival functions ensures a fine-grained estimation of the event distribution, without requiring interpolation techniques~\cite{archetti2023deep,archetti2024bridging}. We theoretically show how this framework makes FPBoost a universal approximator of hazard functions, allowing it to learn in principle any target hazard, provided enough heads. Lastly, gradient-boosted trees applied to tabular data---the most common data format in survival applications---has proven to be still competitive against neural network techniques~\cite{grinsztajn2022tree}, as highlighted by our empirical results.

We evaluate FPBoost in the right-censored, single-event setting, which is the most common application of survival analysis. The performance of FPBoost is benchmarked against state-of-the-art survival models, including both tree-based and neural network-based models~\cite{chen2024introduction}. Performance metrics include the concordance index~\cite{uno2011c} to measure the discrimination capabilities of the model and the integrated Brier score~\cite{graf1999assessment}, tailored for calibration. Our experiments demonstrate that FPBoost outperforms alternative models in both discrimination and calibration in the majority of cases and matches their performance when it does not.

In summary, this study provides the following contributions:
\begin{itemize}
    \item A detailed description of the FPBoost model, motivating the choices behind the composition of its hazard function, and detailing the training procedure based on gradient boosting.
    \item A theoretical analysis on the approximation capabilities of FPBoost showing that it can model, in principle, any target hazard function, provided enough heads.
    \item An extensive empirical analysis providing evidence of our approach's efficacy across a diverse set of datasets and baseline models. The experimental procedure is designed to accommodate the intrinsic variability of survival datasets with low cardinality, providing an accurate estimation of the true generalization performance of each model under consideration.
    \item An open-source Python implementation of FPBoost fully compatible with the \texttt{scikit-survival} library~\cite{sksurv}. This way, the FPBoost algorithm can be directly used as a drop-in replacement in existing pipelines.
\end{itemize}

\section{Background and Related Work}
\label{sec:related-work}

Survival analysis is a field of statistics that focuses on modeling the probability of an event of interest occurring over time for a population. The primary objective of survival models is thus estimating a survival function $S(t|\mathbf{x})$, which measures the probability of surviving, i.e., not experiencing the event up to time $t$ as
\begin{equation*}
    S(t | \mathbf{x}) = P(T > t | \mathbf{x}).
\end{equation*}
Here, $T$ is the time-to-event random variable and $\mathbf{x} \in \mathbb{R}^d$ a $d$-dimensional vector encoding the subject's features. 

The survival function exhibits several key properties. It is monotonically non-increasing, starts at 1 for $t=0$, and asymptotically approaches 0 as $t$ tends to infinity indicating that, given an infinite time frame, all subjects will ultimately experience the event of interest~\cite{klein2003survival}.

A core aspect of survival analysis is the ability to handle censored data. Censoring occurs when subjects do not experience the event of interest within the study period. For these individuals, we can only ascertain that they survived up to a certain time point, but their true event time remains unknown. Consequently, a survival dataset comprises set of $N$ triplets $(\mathbf{x}_i, \delta_i, t_i)$, where:
\begin{itemize}
\item $\mathbf{x}_i$ is the feature vector for subject $i$.
\item $\delta_i$ is a binary indicator of whether the subject experienced the event during the study (1) or was censored (0).
\item $t_i$ is the observed time, corresponding to either the event time or the censoring time.
\end{itemize}
This framework is commonly referred to as right-censored, single-event survival analysis and will be the focus of this work.

\subsection{Non-Parametric Models}

Survival models can be categorized into three groups: non-parametric, semi-parametric, and fully parametric~\cite{wang2019machine}. The group of non-parametric models comprises statistical estimators that provide information about data without any prior assumption on the event distribution. Non-parametric models rely on some notion of similarity between groups of individuals to improve the prediciton complexity. The most common non-parametric model is the Kaplan--Meier (KM) estimator~\cite{kaplan1958nonparametric}, which is often used to plot the general survival behavior of a population. 
% The KM estimator is defined as
% \begin{equation*}
% S(t) = \prod_{t_i \leq t} \left(1 - \frac{e_i}{r_i}\right)
% \end{equation*}
% where $t_i$ are the unique observed event times, $e_i$ is the number of events at time $t_i$, and $r_i$ is the number of subjects at risk before $t_i$. 
In fact, the KM estimator is not conditioned on the subjects'~features as it is tailored to provide aggregate information about the overall event distribution within the data. Another popular non-parametric model is the Random Survival Forest~\cite{ishwaran2008random}, which builds a set of decision trees with the CART~\cite{breiman1984classification} method by maximizing the event distribution difference between nodes according to repeated log-rank tests~\cite{bland2004logrank}. Each leaf contains a non-parametric estimation of the subjects corresponding assigned to that specific terminal node. The final prediction is obtained by averaging the predictions of the trees in the forest.

Finally, BoXHED~\cite{boxhed1,boxhedth,boxhed2} is a nonparametric boosting method for hazard estimation specifically designed to handle time-dependent covariates. Like FPBoost, BoXHED employs gradient boosting to minimize the survival likelihood, but extends beyond static covariates to accommodate time-varying covariate processes. The key difference lies in how the gradient is computed: BoXHED derives its gradient through a smooth convex representation of the likelihood function to accommodate time-varying trajectories, whereas FPBoost leverages the explicit parametric forms of its hazard functions, which promotes interpretable risk components corresponding to known distributions.

\subsection{Semi-Parametric Models}

Semi-parametric models are crucial tools for survival analysis, providing the ability to build survival estimations from a combination of non-parametric and parametric techniques. These models focus on predicting the hazard function, a quantity related to the survival function, which measures the instantaneous risk for subjects that have survived up to time $t$:
\begin{equation*}
    h(t | \mathbf{x}) = \lim_{\Delta t \rightarrow 0} \frac{P(t \leq T < t + \Delta t| T \geq t, \mathbf{x})}{\Delta t}.
\end{equation*}
Differently from the survival function, which is constrained between 0 and 1, the hazard function can take values greater than 1. Additionaly, the survival function is related to the hazard function as
\begin{equation}
    S(t|\mathbf{x}) = \exp(-H(t | \mathbf{x})) = \exp\left(-\int\limits_0^t h(u|\mathbf{x}) du\right)
    \label{eq:h-to-s}
\end{equation}
where $H(t|\mathbf{x})$ is the cumulative hazard function, defined as the integral of the hazard function from 0 to $t$. For this reason, the cumulative hazard diverges for $t\rightarrow\infty$ to allow the survival function to asymptotically approach 0~\cite{wang2019machine}.

The Cox model~\cite{cox1972regression} is a prominent semi-parametric model and serves as a primary baseline for machine learning-based survival analysis. This model relies on two key assumptions: (i) linear dependency between features and risk of experiencing an event (ii) the ratio between hazard functions of different subjects is constant over time. This latter is often referred to as proportional hazard assumption. While potentially limiting when the model is applied to large datasets, these assumptions provide a strong bias, enabling effective generalization even with limited data samples. In particular, the Cox model defines the hazard function as the product of a baseline hazard, $h_0(t)$, and the exponential of a subject-dependent risk factor $r(\mathbf{x})$:
\begin{equation*}
    h(t|\mathbf{x}) = h_0(t)\cdot \exp(r(\mathbf{x})) = h_0(t) \cdot \exp(\beta^T\mathbf{x}),
\end{equation*}
where $h_0(t)$ is a non-parametric hazard estimation common to all samples, such as the Breslow estimator~\cite{breslow1974covariance}. The parameters $\beta$ are trained using the partial log-likelihood loss:
\begin{equation*}
    L = -\frac{1}{N}\sum\limits_{i=1}^N \beta^T\mathbf{x}_i - \log\sum\limits_{j: t_j \geq t_i} \exp(\beta^T \mathbf{x}_j).
\end{equation*}
Several extensions of the Cox model have been proposed, all relying on the proportional hazard assumption and partial log-likelihood optimization.  Among those, CoxBoost~\cite{ridgeway1999state} and XGBoost~\cite{xgboost} optimize the loss using gradient-boosted decision trees to estimate $r(\mathbf{x})$. On the other hand, DeepSurv~\cite{katzman2018deepsurv} replaces the linear dependency between parameters and features with a single-output neural networks.

\subsection{Fully Parametric Models}

Fully parametric survival models estimate the entire survival function using a set of parameters. Historically, these models assumed that the event occurrence followed a particular probability distribution, such as Weibull, LogNormal, or LogLogistic with parameters $\Theta$. Given this assumption, for right-censored single-event survival data, the distribution parameters can be estimated by maximizing the survival likelihood as
\begin{equation}
\hat{\Theta} = \arg\max_{\Theta} \prod\limits_{i=1}^N h(t_i|\Theta)^{\delta_i}S(t_i|\Theta).
\label{eq:lik}
\end{equation}
Building upon standard fully parametric distributions, Deep Survival Machines (DSM)~\cite{nagpal2021deep} propose a parameter estimation neural network to construct a mixture of predefined probability distributions. The final survival function is then computed as a weighted sum of these distributions. DSM is trained using a combination of ELBO losses and a regularization prior loss in a Bayesian framework.

Other popular models do not rely on predefined statistical distributions to construct survival estimations, but leverage neural networks to estimate the event probability directly at a fixed set of time intervals~\cite{kvamme2019time,kvamme2021continuous}. These neural networks have a single output per time bin, representing the event probability for that interval. One such model is DeepHit~\cite{lee2018deep}, a discrete-time survival model consisting of a shared feature extractor followed by event-specific sub-networks estimating the probabilities for each event. While these discrete-time models have shown promising practical results, they struggle with fine-grained or long-term prediction horizons due to their fixed-time nature. To address this limitation, some studies have proposed interpolation techniques between time points~\cite{kvamme2021continuous,archetti2024bridging}.

\section{Fully Parametric Gradient Boosting}
\label{sec:fpboost}

\begin{figure*}[t]
    \centering
    \includegraphics[width=\textwidth]{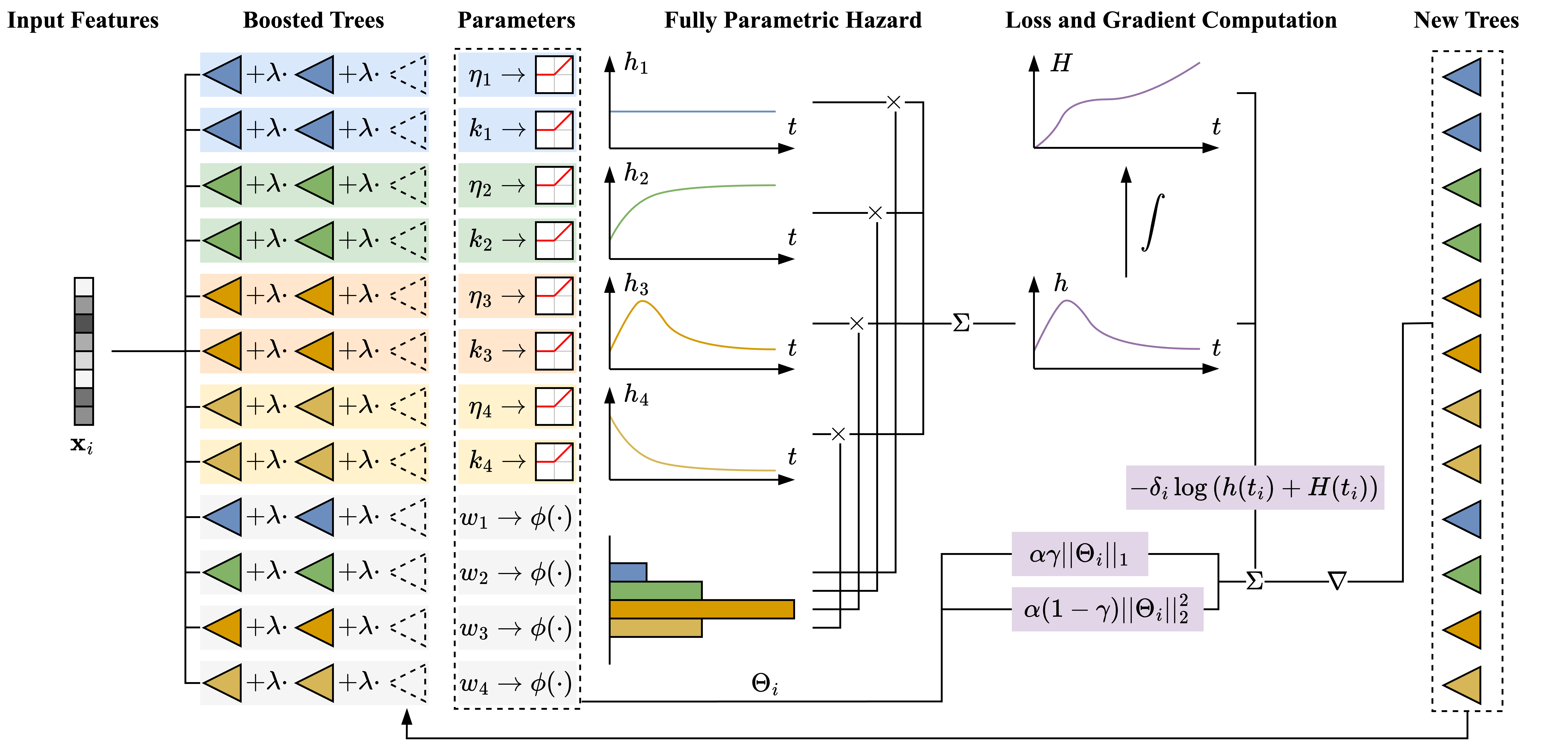}
    \caption{FPBoost architecture example with four heads. A set of trees estimates two distribution parameters, $\eta_j$ and $k_j$, for each of four heads starting from the input features. Heads 1 ({\color{NavyBlue}{blue}}) and 2 ({\color{ForestGreen}{green}}) follow Weibull distributions, while heads 3 ({\color{orange}{orange}}) and 4 ({\color{Dandelion}{yellow}}) follow LogLogistic distributions. An additional set of trees ({\color{Gray}{gray}}) estimates a weight for each head. These heads are combined to form a single hazard function and its corresponding cumulative hazard function. New trees are built by fitting the gradient of the negative log-likelihood and ElasticNet ({\color{Purple}{purple}}).}
    \label{fig:fpboost}
\end{figure*}

FPBoost is a novel survival algorithm based on the weighted sum of fully parametric hazard functions. Parameter estimation is carried out via gradient boosting, optimizing the negative log-likelihood loss function.

\subsection{Hazard Function Definition}

In FPBoost, the hazard function is composed of $J$ heads, each corresponding to a parametric distribution. This work includes heads following either the Weibull or LogLogistic distributions but, in principle, any hazard distribution differentiable with respect to its parameters can be included as a head. Table~\ref{tab:distributions} collects the formulation of the considered distributions. This choice comes primarily from the fact that Weibull distributions are suited to model constant and increasing risks, while LogLogistic ones provide decreasing and arc-shaped risk profiles. Typically, the former is associated to survival behaviors such as aging and wear, while the latter relates to infant mortality. When combined, the resulting risk can assume a bathtub shape, covering a wide range of real-world behaviors~\cite{klein2003survival,nagpal2021deep}. This approach provides an advantage from both a learning and interpretation perspective, as each sample is associated with a set of well-known distribution functions, where each parameter has a clear interpretation. In particular, for both the considered distributions, the $\eta$ parameter controls the scale (or spread) while the $k$ parameter changes the curve shape by modifying the slope or adding a bump.

Figure~\ref{fig:fpboost} depicts an example of a 4-headed FPBoost architecture with heads 1 and 2 following a Weibull distribution and heads 3 and 4 a LogLogistic distribution. The distribution parameters are estimated from the input features $\mathbf{x}$ using a set of regression trees. Additionally, another set of trees estimates a weight $w_j$ for each head.
Thus, the hazard function of FPBoost is defined as
\begin{equation*}
h(t|\Theta) = \sum\limits_{j=1}^J w_j h_j(t | \eta_j,k_j)
\end{equation*}
where $w_j$ are the learned head weights and $\eta_j,k_j$ are the parameter estimations from the regression trees for the $j$-th head. To improve readability, we define $\Theta$ as the vector containing all the distribution parameters $\eta_j,k_j$ and head weighting factors $w_j$ for each of the $J$ heads. 

\input{tab_distributions}

To guarantee the validity of the Weibull and LogLogistic hazard formulations, the scale and shape parameters $\eta_j,k_j$ must be nonnegative. To enforce this constraint, we apply a ReLU activation function~\cite{nair2010rectified} to these parameters, ensuring that any negative estimates are set to zero.

Similarly, the weighting parameters $w_j$ are processed by an activation function $\phi(\cdot)$ before hazard computation. The choice of activation function significantly influences model interpretability and generalizability. In particular, functions that yield nonnegative values, such as ReLU, sigmoid, and softmax, promote interpretability, as they quantify the contribution of each hazard component with a positive weight. This approach allows for meaningful insights into the relative contributions of Weibull and LogLogistic heads in hazard estimation. For instance, a predominance of Weibull-based hazards could indicate a higher susceptibility to aging-related failures, whereas a stronger presence of LogLogistic hazards might suggest risks associated with early failure or infant mortality.

% To ensure that the distribution parameters and the final hazard are nonnegative, we apply a ReLU activation function to each $\eta_j$ and $k_j$, setting negative values to zero~\cite{nair2010rectified}. The weighting parameters $w_j$ are similarly processed by an activation function before hazard computation, with the specific choice depending on the requirements of the domain. Activation functions that yield nonnegative values, such as ReLU, sigmoid, and softmax, are preferred for their interpretability, as they ensure that each hazard's contribution is positively weighted. Insights gained from this approach might include the relative influence of Weibull versus LogLogistic heads on hazard predictions. For example, a Weibull predominance could suggest a vulnerability to aging processes, whereas a stronger presence of LogLogistic hazards might indicate risks typical of early failure, or infant mortality.

The use of activation functions $\phi(\cdot)$ that allow for negative weights, despite making interpretation less straightforward, enables a a broader modeling capacity, as demonstrated by the following result:
\begin{theorem}%[Universal Approximation via Weibull Hazards]
\label{thm:weibull_universal_approx}
Let $\mathcal{H}$ denote the space of hazard functions, that is, continuous nonnegative real functions $h(t)$ for which $\int_0^\infty h(t)\,dt=\infty$. 
For any $h^\star \in \mathcal{H}$, any $\varepsilon > 0$, and any interval $[0, T]$, there exists a finite collection of $J$ Weibull hazard functions $h^W_j(t)$, with parameters 
$\eta_j,k_j$, and weights $w_j$, such that
\[
\sup_{t \in [0,T]} \biggl|\,h^\star(t) \;-\; \sum\limits_{j=1}^J w_j\,h^W_j(t)\biggr| \;<\; \varepsilon.
\]
\end{theorem}
The proof, detailed in Appendix~\ref{sec:proofs}, is a direct consequence of the Weierstrass Approximation Theorem, after showing that a single Weibull head is equivalent to a monomial of arbitrary degree. This underpins the capability of FPBoost to approximate any target hazard function with sufficient heads. It is important to note that while only Weibull hazards are required for this theoretical result, incorporating LogLogistic hazards often enhances model performance and makes optimization easier in practical applications.

On a final technical note, allowing negative weights means the final hazard function could potentially be negative for some values of $t$. Although this does not compromise the theoretical result, the FPBoost implementation prevents this issue by clipping the final hazard value to be above zero.

\subsection{Loss Function and Training Procedure}

Since the FPBoost hazard is differentiable with respect to the distribution parameters and weights, the estimators can be trained by minimizing the full negative log-likelihood loss function, which derives from equations~\ref{eq:h-to-s} and~\ref{eq:lik}, without the need of simplifying assumptions:
\begin{equation*}
L_{\text{lik}} = -\frac{1}{N}\sum\limits_{i=1}^N \delta_i \log (h(t_i | \Theta_i)) - H(t_i | \Theta_i).
\end{equation*}
In order to prevent overfitting, we add to $L_{\text{lik}}$ an ElasticNet regularization term~\cite{zou2005regularization} as
\begin{equation*}
    L_{\text{reg}} = \alpha \left(\gamma \left\| \Theta \right\|_1 + (1 - \gamma) \left\| \Theta \right\|_2^2\right)
\end{equation*}
where $\alpha \geq 0$ weighs the regularization contribution to the final loss and $\gamma \in [0, 1]$ controls the ratio between the $L1$ and $L2$ penalties. 

Training is performed using a standard gradient boosting algorithm. Specifically, for each parameter from $\Theta$, being it a scale, shape, or weight, an empty list of trees is initialized. Then, at iteration $m$, each of these lists is populated by a new tree $\tau_j^{(m)}(\mathbf{x})$ fitted on the negative gradient of the loss function, called pseudo-residual. Parameter estimation occurs by summing the estimations of all trees belonging to each list, weighted by a learning rate $\lambda > 0$~\cite{friedman2001greedy}. This way, each iteration produces a set of trees which contribute to loss minimization and consequently to parameter refinement, in a standard gradient boosting fashion. The pseudocode of FPBoost is provided in Algorithm~\ref{alg:fpboost}.

\begin{algorithm}[t]
\caption{FPBoost Training}
\label{alg:fpboost}
\begin{small}
\begin{algorithmic}[1]
\REQUIRE Training data $\{(\mathbf{x}_i, \delta_i, t_i)\}_{i=1}^N$, number of heads $J$, distribution types for each head (Weibull or LogLogistic), number of iterations $M$, learning rate $\lambda$, regularization parameters $\alpha$ and $\gamma$, weight activation function $\phi(\cdot)$

\STATE \textbf{Initialize} $3J$ parameter models with random values: \\
\centerline{${\bigl\{ F_{\eta_j}^{(0)}(\mathbf{x}),\; F_{k_j}^{(0)}(\mathbf{x}),\; F_{w_j}^{(0)}(\mathbf{x}) \bigr\}_{j=1}^J}$}

\FOR{$m = 0$ to $M - 1$}
    \FOR{$j = 1$ to $J$}
        \STATE $\eta_j = \mathrm{ReLU}\bigl(F_{\eta_j}^{(m)}(\mathbf{x})\bigr)$
        \STATE $k_j = \mathrm{ReLU}\bigl(F_{k_j}^{(m)}(\mathbf{x})\bigr)$
        \STATE $w_j = \phi\bigl(F_{w_j}^{(m)}(\mathbf{x})\bigr)$
    \ENDFOR
    \STATE Collect parameters as $\Theta = \{\eta_j, k_j, w_j\}_{j=1}^J$
    \STATE $L_{\mathrm{lik}}=-\frac{1}{N}\sum_{i=1}^{N}\Bigl[
      \delta_i \log \bigl(h(t_i \mid \Theta_i)\bigr)
      - H(t_i \mid \Theta_i)
    \Bigr]$

    \STATE $L_{\mathrm{reg}} = \alpha \Bigl(
      \gamma \|\Theta\|_1 + (1-\gamma)\|\Theta\|_2^2
    \Bigr)$

    \STATE $L \;=\; L_{\mathrm{lik}} \;+\; L_{\mathrm{reg}}$

    \FOR{\textbf{each} parameter model $F_{\cdot_j}^{(m)}(\mathbf{x})$}
        \STATE Compute pseudo-residuals
        ${r_{\cdot_j} =
          -\frac{\partial L\bigl(\mathbf{x}, \delta, t, \Theta\bigr)}
                  {\partial F_{\cdot_j}^{(m)}(\mathbf{x})}}$
        \STATE Fit a regression tree $\tau_{\cdot_j}^{(m)}$ to $r_{\cdot_j}$
        \STATE Update the parameter model: \\
        \centerline{$F_{\cdot_j}^{(m+1)}(\mathbf{x}) = F_{\cdot_j}^{(m)}(\mathbf{x}) + \lambda\;\tau_{\cdot_j}^{(m)}(\mathbf{x})$}
    \ENDFOR
\ENDFOR

\STATE \textbf{Return} parameter models $\bigl\{ F_{\eta_j}^{(M)},F_{k_j}^{(M)},F_{w_j}^{(M)} \bigr\}_{j=1}^J$
\end{algorithmic}
\end{small}
\end{algorithm}

\section{Experiments}
\label{sec:experiments}

This section covers the experimental setup to evaluate the performance of FPBoost alongside the set of baseline survival models.

\subsection{Datasets}
\label{sec:datasets}

To ensure fair evaluation and consistency with similar studies, we selected datasets from well-known benchmarks in survival analysis, covering different conditions like breast cancer, lung cancer, AIDS, and cardiovascular diseases. Specifically, the AIDS~\cite{whas}, Breast Cancer~\cite{breastcancer}, FLCHAIN~\cite{dispenzieri2012use}, GBSG2~\cite{gbsg} and Veterans~\cite{veterans} datasets are provided by the \texttt{scikit-survival}~\cite{sksurv} Python library. The METABRIC~\cite{katzman2018deepsurv} and WHAS~\cite{whas} datasets, instead, are available in the DeepSurv repository~\cite{katzman2018deepsurv} with a predefined test set. Lastly, SUPPORT2~\cite{support} is provided by \texttt{SurvSet}~\cite{drysdale2022survset}. Details on data collection and content overview are provided in Appendix~\ref{sec:data_details}. Table~\ref{tab:datasets} collects the summary statistics of these datasets.

\input{tab_data}

\subsection{Metrics}
\label{sec:metrics}

We evaluated survival models using the concordance index (C-Index) and the integrated Brier score (IBS). The C-Index~\cite{uno2011c} measures the predictive accuracy of survival models by evaluating the proportion of concordant pairs relative to all comparable pairs within a dataset. A pair of subjects $i$ and $j$ is considered comparable if, given $t_i < t_j$, then $\delta_i = 1$. A pair of comparable subjects is concordant when the predicted mean time aligns with the actual event times.

The Brier score~\cite{graf1999assessment} assesses the calibration of probability estimates over time by computing the weighted squared difference between the binary survival indicator of a subject and the predicted survival probability. The Brier score at time $t$ is defined as:
\begin{equation*}
    \text{BS}(t) = \frac{1}{N}\sum\limits_{i=1}^N w_i(t)(\mathbf{1}(t_i > t) - S(t|\mathbf{x}_i))^2,
\end{equation*}
where $\mathbf{1}(\cdot)$ is an indicator function and $w_i(t)$ is a weighting factor that adjusts the censoring bias. This adjustment is the Inverse Probability of Censoring Weighting (IPCW)~\cite{robins1992recovery,uno2011c}, which assigns weights based on the inverse probability of censoring at a given time $t$.
% as
% \begin{equation*}
% w_i(t) = \begin{cases}
% \delta_i / G(t_i) & \text{if } t_i \leq t \\
% 1 / G(t_i) & \text{if } t_i > t
% \end{cases}.
% \end{equation*}
% where $G(t)$ represents the Kaplan-Meier estimate of the censoring distribution, calculated over the dataset with inverted censoring indicators $\delta$. 
The overall calibration of a survival model over time is summarized by integrating the Brier score across the entire study period, yielding the Integrated Brier Score (IBS).

\subsection{Experimental Procedure}
\label{sec:experimental-procedure}

Each dataset is split into a training and test set, with the latter accounting for 20\% of the total samples. For datasets coming from the DeepSurv repository, we employ the provided train-test split while for the others we apply a seeded stratified split on the censoring variable. To ensure a robust evaluation, each experiment is run 30 times with different seeds and all subsequent measurements are averaged over all executions. During each of these executions, the training set is further divided into training and validation, to allow for hyperparameter tuning and early stopping. Before training, standard normalization and one-hot encoding are applied to numerical and categorical features, respectively.

In order to validate FPBoost, we compared its performance against several baseline models. For these models, we employ default parameters provided by implementations in Python libraries. Specifically, we utilize \texttt{scikit-survival} for RSF, Cox, and CoxBoost, \texttt{pycox}~\cite{kvamme2019time} for DeepSurv and DeepHit, \texttt{xgboost}~\cite{xgboostdoc} for XGBoost, and \texttt{auton-survival}~\cite{nagpal2022auton} for DSM. 
Following~\cite{katzman2018deepsurv}, neural network architectures for DeepSurv, DeepHit, and DSM comprise three layers with neuron counts of 3, 5, and 3 times the number of features, respectively.

For FPBoost, we conduct a hyperparameter search for each dataset, selecting the model with the highest mean C-Index. The search, performed using random search, explores a space including the number of Weibull heads ($\{0, \dots, 32\}$), the number of LogLogistic heads ($\{0, \dots, 32\}$), the number of gradient-boosted trees per parameter ($\{1, \dots, 512\}$), the maximum tree depth ($\{1, \dots, 6\}$), the weights activation functions (ReLU, sigmoid, softmax, hyperbolic tangent, or identity), the boosting learning rate ($[0.01, 1]$), the and ElasticNet loss parameters (${\alpha\in [0, 1]}$ and ${\gamma\in [0, 1]}$). To prevent fully parametric distributions to be affected by different time scales, we normalized the time values for each dataset between 0 and 1. The best hyperparameters for each dataset are provided in Appendix~\ref{sec:hparams}.

The source code for these experiments is available at \url{https://github.com/archettialberto/fpboost}. The \texttt{FPBoost} class implementation is fully compatible with \texttt{scikit-survival}, facilitating its inclusion into existing codebases.

\section{Results}
\label{sec:results}

\input{tab_cid}
\input{tab_ibs}

This section presents and analyzes the empirical evaluation of FPBoost against classical and state-of-the-art survival models described in Section~\ref{sec:related-work}, specifically RSF, Cox, CoxBoost, DeepSurv, DSM, and DeepHit. Tables~\ref{tab:cid} and~\ref{tab:ibs} report the performance of each model according to the C-Index and IBS metrics, respectively. Additional results and metric summaries across model types are provided in Appendix~\ref{sec:add_results}. For improved readability, all results and metric reports are scaled up by a factor of 100. 

The C-Index results in Table~\ref{tab:cid} demonstrate FPBoost's competitive performance across all datasets, outperforming other models in all cases except for AIDS and Veteran datasets, where it is marginally surpassed by RSF and DeepSurv. Averaging across all the datasets, FPBoost improves the C-Index against the baseline score by 4.6 points. Since, by definition, a random guessing model has a C-Index of 50 and the C-Index metric is evaluated in a $[50,100]$ range, the performance gain of FPBoost on said metric approximately corresponds to a 9\% improvement. When compared specifically to semi-parametric models (Cox, CoxBoost, and DeepSurv), FPBoost's average improvement is 4.1, highlighting the potential limitations of the proportional hazard assumption in capturing complex data patterns. However, the performance gap in favor of FPBoost becomes more pronounced when compared to neural network-based models (DeepSurv, DSM, and DeepHit), showing an average improvement of 5.5. This suggests that neural networks, despite their capacity to learn complex non-linear patterns, may require more sophisticated tuning to prevent overfitting and ensure strong generalization. In such cases, simplifying assumptions can be beneficial in introducing bias, as suggested by DeepSurv surpassing both DeepHit and DSM. Notably, the average improvement of FPBoost over RSF is smaller at 1.5, indicating that non-parametric algorithms with minimal assumptions may be better suited for tree-based learners compared to neural networks.

Table~\ref{tab:ibs} demonstrates the calibration performance of FPBoost according to the IBS metric, corroborating the trends analyzed on C-Index. Here, FPBoost always ranks first or second, with the exception of the AIDS dataset. Averaging across all datasets, the improvement in IBS of FPBoost is 2.8. Since, by definition, a random guessing model has an IBS of 25 and the IBS metric is evaluated in a $[0,25]$ range, the raw score improvement translates to an approximate 11\% improvement. We opted for excluding XGBoost from these calculations, given its outlier performance with respect to the alternatives. Consistently with the C-Index results, the IBS improvement relative to proportional hazard models is 1.7, increasing to 4.5 for neural network-based models, while the difference with tree-based models is smaller at 0.7. These findings are in line with C-Index, indicating that model performance on survival concordance is reflected also on probability calibration.

In summary, these empirical results showcase the competitive performance of FPBoost against various classical and state-of-the-art models, both tree-based and neural network-based. Performance improvements are evident in terms of both concordance and calibration with respect to neural alternatives. These findings suggest that tree-based nature of FPBoost, combined with direct optimization of the survival likelihood, represents a promising approach for developing more complex, competitive, and adaptable survival models.

\subsection{Discussion and Future Work}

While the idea of leveraging a mixture of parametric functions~\cite{nagpal2021deep} has been previously explored in the survival literature, as well as ensemble learning~\cite{ishwaran2008random,archetti2023scaling}, FPBoost introduces several innovations.

Firstly, the weighted sum is applied directly to the hazard function, unlike previous works that applied it to the survival function. The advantage of this formulation is twofold. On the one hand, summation on Weibull hazards guarantees universal approximation, provided enough heads. 
On the other hand, a direct comparison with DSM suggests that learning parameters directly in hazard-space rather than explicitly weighting survival distributions can be beneficial for a more effective training.

% On the other hand, given the prominence and popularity of models based on hazard assumptions, such as proportionality or accelerated failure time, we argue that directly learning parameters in hazard-space can be beneficial for a more effective training.

Secondly, FPBoost directly maximizes the survival likelihood, without relying on simplified custom loss functions such as the partial likelihood of proportional models or discrete loss functions of neural-network-based models. This is possible due to the assumption that the global hazard function is a composition of differentiable parametric hazard functions. This aspect, combined with the tree-based nature of the algorithm, contributed the most to the empirical results obtained. In fact, ensembles of decision tree have been historically extremely effective in dealing with tabular data so far, even against neural networks and deep learning~\cite{grinsztajn2022tree}.

Future work could explore the theoretical bounds and practical limits of the approximation capabilities of FPBoost, building upon the results presented in this study to improve its mathematical grounding. Another aspect worth investigating is the application of FPBoost to competing risks scenarios, for example by delegating separate sets of heads to different events of interest, as in DeepHit. On top of that, investigating model performance on larger, more diverse and multimodal datasets beyond the healthcare context could further validate its practical utility. Finally, its inclusion in federated learning scenarios can be beneficial for applications where data scarcity and privacy hinder the results and applicability of existing models.

\section{Conclusion}
\label{sec:conclusion}

In this study, we introduced FPBoost, a model for survival analysis that leverages a weighted sum of parametric hazard functions optimized through gradient boosting. Our approach addresses several limitations of existing models by avoiding restrictive assumptions such as proportional hazards, accelerated failure time, or discrete time estimations. On top that, FPBoost is proven to be a universal approximator of hazard functions, guaranteeing maximum modeling flexibility. 
The extensive evaluation of FPBoost across diverse datasets demonstrated its competitive concordance and calibration performance compared to classical and state-of-the-art survival models, including both tree-based and neural network-based approaches. These results highlight the potential of combining parametric hazard functions with ensemble learning techniques in survival analysis, alongside direct optimization of the survival likelihood.

\section*{Impact Statement}

FPBoost has the potential to significantly improve survival estimates, outperforming current methods and offering improved reliability. This advancement can help physicians make more informed decisions by providing robust insights that complement clinical expertise. Its application supports better risk assessment, a cornerstone of data-driven medicine that enables early intervention, personalized treatment and optimal resource allocation to improve patient outcomes and healthcare efficiency. We emphasize, however, that statistical models should not replace expert judgment, but rather serve as complementary tools that reinforce a data-driven yet ethically responsible approach to risk assessment.

\section*{Acknowledgements}
\label{sec:acknowledgements}
This paper is supported by the FAIR (Future Artificial Intelligence Research) project, funded by the NextGenerationEU program within the PNRR-PE-AI scheme (M4C2, investment 1.3, line on Artificial Intelligence

\bibliography{references}
\bibliographystyle{icml2025}

\newpage
\clearpage
\appendix
% \onecolumn

\section{Appendix}

\subsection{Proofs}
\label{sec:proofs}

\begin{theorem}%[Universal Approximation via Weibull Hazards]
Let $\mathcal{H}$ denote the space of hazard functions, that is, continuous nonnegative real functions $h(t)$ for which $\int_0^\infty h(t)\,dt=\infty$. 
For any $h^\star \in \mathcal{H}$, any $\varepsilon > 0$, and any interval $[0, T]$, there exists a finite collection of $J$ Weibull hazard functions $h^W_j(t)$, with parameters 
$\eta_j,k_j$, and weights $w_j$, such that
\[
\sup_{t \in [0,T]} \biggl|\,h^\star(t) \;-\; \sum\limits_{j=1}^J w_j\,h^W_j(t)\biggr| \;<\; \varepsilon.
\]
\end{theorem}
\begin{proof}[Proof]
Let $C[0,T]$ be the set of continuous real functions in the interval $[0,T]$ and $\mathcal{H}_T$ the set of functions in $\mathcal{H}$ restricted to the interval $[0,T]$. By construction, $\mathcal{H}_T \subset C[0,T]$. By the Weierstrass Approximation Theorem~\cite{jeffreys1988weierstrass}, for any $\varepsilon > 0$ and any function $f \in C[0,T]$, there exists a polynomial 
\[
P(t) = \sum_{n=0}^N a_n\,t^n,
\]
such that
\[
\sup_{t \in [0,T]} \bigl|f(t) - P(t)\bigr| < \varepsilon.
\]
Hence, we reduce the problem of approximating $h^\star$ uniformly on $[0,T]$ 
to approximating the polynomial $P$. To do this, recall that a Weibull hazard function with parameters $\eta, k$ is given by
\[
h^W(t) \;=\; \eta \, k \, t^{k-1}.
\]
For integer $n = k-1 \ge 0$, $\eta > 0$, and $b=\eta \, (n+1)$, this 
becomes
\[
h^W(t) \;=\; b\,t^n,
\]
which is a monomial in $t$ of degree $n$ with constant multiplicative factor $b$. 
Given the polynomial $P(t)$, consider a weighted sum of $N$ Weibull hazards. By choosing weights $w_n = a_n/b_n$ to match the coefficients of $P(t)$, we have
\[
\sum_{n=0}^N w_n\,h^W_n(t) \;=\; P(t).
\]
Thus, by the Weierstrass argument,
\[
\sup_{t \in [0,T]} \biggl|\,f(t) - \sum_{n=0}^N w_n\,h^W_n(t)\biggr| < \varepsilon.
\]
Since this bound holds for any function $f\in C[0,T]$ and $\mathcal{H}_T \subset C[0,T]$, then the bound holds also for any $h^\star \in \mathcal{H}_T$.
\end{proof}

\subsection{Dataset Details}
\label{sec:data_details}

\begin{figure*}[t]
    \centering
    \includegraphics[width=\textwidth]{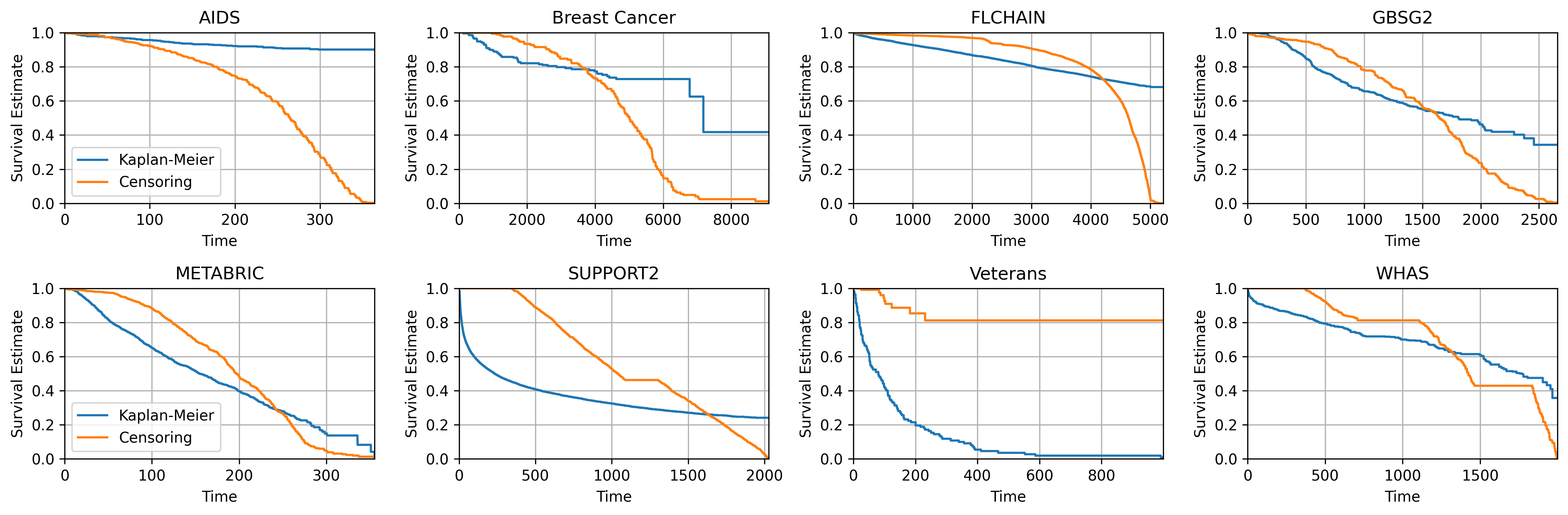}
    \caption{Kaplan--Meier estimations ({\color{NavyBlue}{blue}}) on survival probability and censoring probability ({\color{Orange}{orange}}) for the datasets included in the study.}
    \label{fig:curves}
\end{figure*}

\input{tab_hparams}

The AIDS~\cite{whas} dataset originates from a trial comparing three-drug and two-drug regimens in HIV-infected patients. The primary event of interest was the time to an AIDS-defining event or death. The high censoring percentage resulted from the trial being terminated early after reaching a predefined level of statistical significance.

The Breast Cancer~\cite{breastcancer} dataset is derived from a study aimed at validating a 76-gene prognostic signature for predicting distant metastases breast cancer patients. The study includes gene expression profiling of frozen samples from 198 patients. 

The FLCHAIN~\cite{dispenzieri2012use} dataset originates from a study examining the relationship between serum free light chains and mortality in a general population cohort. It includes data from 7874 subjects. The primary endpoint is death, which occurred in 2169 subjects (27.5\%).

The German Breast Cancer Study Group (GBSG2)~\cite{gbsg} dataset targets breast cancer recurrence post-treatment, evaluating hormone therapy's impact on recurrence. Collected from a randomized study in Germany, it includes covariates like age, menopausal status, tumor size, and node status. 

The Molecular Taxonomy of Breast Cancer International Consortium (METABRIC)~\cite{katzman2018deepsurv} aims to understand breast cancer through molecular taxonomy to develop personalized treatments based on tumor genetic profiles. The dataset encompasses a mix of clinical features and genomic data, with a patient cohort from Canada and UK. 

The Study to Understand Prognoses Preferences Outcomes and Risks of Treatment (SUPPORT2)~\cite{support} focuses on critically ill hospitalized patients, conducted in two phases between 1989 and 1997. It covers administrative and clinical follow-up for six months post-inclusion in the study. The version of this dataset used in the experiments includes 35 features.

The Veteran Administration Lung Cancer Trial (Veterans)~\cite{veterans} dataset focuses on lung cancer patients treated with two different chemotherapy regimens. This dataset is frequently used in simple survival benchmarks due to its small sample size. 

The Worcester Heart Attack Study (WHAS)~\cite{whas} deals with cardiovascular health, tracking 1638 patients post-myocardial infarction from 1997 to 2001. It includes biometric parameters and temporal features like hospital stay length and follow-up dates.

Figure~\ref{fig:curves} shows the Kaplan--Meier estimations on survival probability and censoring probability for the datasets included in the study.

\subsection{Hyperparameter Tuning}
\label{sec:hparams}

\input{tab_ctd}
\input{tab_auc}

The implementation of FPBoost allows to tune the following hyperparameters:
\begin{itemize}
\item \textsc{Estimators}: Maximum number of gradient-boosted trees per estimated parameter. Search values are in the interval $\{1, \dots, 512\}$.
\item \textsc{Weibull}: Number of Weibull heads to include in the architecture. Search values are in $\{0, \dots, 32\}$.
\item \textsc{LogLogistic}: Number of Weibull heads to include in the architecture. Search values are in $\{0, \dots, 32\}$.
\item \textsc{Max Depth}: Maximum decision tree depth, according to the regression tree implementation from \texttt{scikit-learn}. Search values are $\{1, 3, 6\}$.
\item $\lambda$: Learning rate for the gradient boosting algorithm weighting the contribution of each tree. Search valued are in $[0.01, 1]$.
\item $\alpha$: Scaling factor for the ElasticNet loss. Search values are in $[0, \dots, 1]$.
\item $\gamma$: Ratio between $L1$ and $L2$ penalties in ElasticNet. Search values are in $[0, \dots, 1]$.
\item $\phi(\cdot)$: Activation function to apply to the estimated head weights $w_j$. Allowed functions are ReLU, softmax, sigmoid, hyperbolic tangent, and identity.
\item \textsc{Initialization}: How to initialize parameter estimators $F_{\cdot_j}^{(0)}$. If \textsc{Random}, parameters are initialized as $\eta\sim\mathcal{N}(0.5, 1)$, $k\sim\mathcal{N}(0,2)$, $w\sim\mathcal{N}(0,1)$. If \textsc{KM}, instead, parameter initialization is based on the Kaplan--Meier estimator. Specifically, a Weibull and a LogLogistic distribution are fitted to the KM estimator, obtaining $\bar{\eta},\bar{k}$ for both distributions. Then, head parameters are initialized as $\eta\sim\mathcal{N}\left(\bar{\eta},\bar{\eta}/10\right)$ and $k\sim\mathcal{N}\left(\bar{k},\bar{k}/10\right)$. In the \textsc{KM} case, weights are initialized uniformly.
\item \textsc{Patience}: Used to stop training before the execution of $M$ iterations if the validation C-Index does not increase for \textsc{Patience} rounds.
\end{itemize}
Table~\ref{tab:hparams} collects the best hyperparameters selected for each dataset.
Hyperparameters have been optimized on the validation score with random search allowing 1024 maximum trials per dataset. The experiments ran on a machine equipped with an Intel(R) Xeon(R) Gold 6238R CPU @ 2.20GHz with 256GB of RAM running Ubuntu 20.04.6 LTS. 

\subsection{Additional Results}
\label{sec:add_results}

\subsubsection{Results on C-TD and AUC}

\input{tab_mean}

This section provides additional results on Time-Dependent Concordance index (C-TD) and AUC~\cite{sksurv}. Specifically, the C-TD computes a concordance index weighting the contribution of each sample with the inverse censoring probability (IPCW)~\cite{robins1992recovery,uno2011c} following the same procedure as the IBS metrics. This reduces the bias in the concordance measure introduced by skewed censoring occurrences. Table~\ref{tab:ctd} collects the experiment results on the C-TD metrics.

Similarly to concordance, the cumulative area under the receiver operating characteristic curve (AUC) evaluates the ability to distinguish between subjects who experience an event by a specific time \( t \) and those who do not. For a given risk score estimation \( r_i \), the AUC at time \( t \) is defined as:
\[
\mathrm{AUC}(t) = \frac{\sum_{i,j} w_i\cdot\mathbf{1}(t_j > t \wedge t_i \leq t) \cdot\mathbf{1}(r_j \leq r_i)} {(\sum_{i} \mathbf{1}(t_i > t)) (\sum_{i} w_i\cdot\mathbf{1}(t_i \leq t))}.
\]
Here, \( \mathbf{1} \) is the indicator function, \( t_i \) denotes the observed time for subject \( i \), and \( w_i \) represents the IPCW. Integrating this metric over time results in a measure of discriminative performance in the presence of censored data. Table~\ref{tab:ctd} collects the experiment results on the AUC metrics.

Finally, Table~\ref{tab:avg} reports the metrics averaged over all datasets, resulting in an aggregated view of the overall model performance on a diverse set of benchmarks.

We opted not to include these additional results in the main analysis as they largely align with the trends observed in the C-index and IBS, which are the most common metrics in survival model assessment, making them somewhat redundant. However, they offer further evidence of FPBoost's consistency across different evaluation criteria, reinforcing the practical reliability its predictions.

\subsubsection{The Impact of Model Type}

While the adoption of FPBoost is generally beneficial in terms of survival metrics, the entity of improvement depends on the type of benchmark considered. As discussed in the results from Section~\ref{sec:results}, the improvement is more noticeable when considering survival models based on neural networks, and shrinks significantly for non-parametric models. Here we report the aggregated comparison over two dimensions, survival model type and estimator type. 

As stated in Section~\ref{sec:related-work}, survival models can be categorized into three types, non-parametric (RSF), semi-parametric (Cox, CoxBoost, XGBoost, DeepSurv), and fully parametric (DeepHit, DSM). Similarly, survival estimations from subject features can have a linear (Cox) or non-linear dependency, as a result of decision trees (CoxBoost, XGBoost, RSF) and neural networks (DeepSurv, DeepHit, DSM). 

Table~\ref{tab:models} summarizes the survival and estimation type of each model considered in this study. Additionally, Tables~\ref{tab:cid_summary}, \ref{tab:ibs_summary}, \ref{tab:ctd_summary}, and \ref{tab:auc_summary} provide the aggregated view of FPBoost improvements with respect to the baselines.

\input{tab_models}
\input{tab_cid_summary}
\input{tab_ibs_summary}
\input{tab_ctd_summary}
\input{tab_auc_summary}

\end{document}

%% file: tab_distributions.tex
\begin{table}[t]
\caption{Analytical expression of the hazard and cumulative hazard of the Weibull and LogLogistic distributions. Both of these distributions depend on two parameters, a scale parameter $\eta$ and a shape parameter $k$.}
\label{tab:distributions}
\vskip 0.15in
\begin{center}
\begin{small}
\begin{sc}
\begin{tabularx}{\linewidth}{@{}lYc@{}}
\toprule
Distribution & Hazard & Cumulative Hazard \\
\midrule
Weibull & $\eta k t^{k-1}$ & $\eta t^k$ \\
LogLogistic & $\eta kt^{k-1}/(1+\eta t^k)$ & $\log(1+\eta t^k)$\\
\bottomrule
\end{tabularx}
\end{sc}
\end{small}
\end{center}
\vskip -0.1in
\end{table}

%% file: tab_data.tex
\begin{table}[t]
\caption{Summary statistics of the survival datasets involved in the experiments.}
\label{tab:datasets}
\vskip 0.15in
\begin{center}
\begin{small}
\begin{sc}
\begin{tabularx}{\linewidth}{@{}lYYY@{}}
\toprule
Dataset & Samples & Censoring & Features \\
\midrule
AIDS & 1151 & 91.66\% & 11 \\
Breast Cancer & 198 & 74.24\% & 80 \\
FLCHAIN & 7874 & 72.45\% & 9 \\
GBSG2 & 686 & 56.41\% & 8 \\
METABRIC & 1904 & 42.07\% & 9 \\
SUPPORT2 & 9105 & 31.89\% & 35 \\
Veterans & 137 & 6.57\% & 6 \\
WHAS & 1638 & 57.88\% & 6 \\
\bottomrule
\end{tabularx}
\end{sc}
\end{small}
\end{center}
\end{table}

%% file: tab_cid.tex
\begin{table*}[t]
\caption{Test C-Index ($\uparrow$) and 95\% confidence interval for each model and dataset, averaged across 30 seeded splits (same test, different train and validation sets). To enhance readability, all values are scaled by a factor of 100. Best results are highlighted in \textbf{bold}, while the second best are \underline{underlined}. 
% Results with $^*$ indicate no statistically significant difference between FPBoost and the respective model according to a t-test at the 95\% confidence level.
}
\label{tab:cid}
\vskip 0.15in
\begin{center}
\begin{small}
\begin{sc}\begin{tabularx}{\linewidth}{@{}lYYYYYYYY@{}}
\toprule
% Model & AIDS & Breast & FLCHAIN & GBSG2 & METABRIC & SUPPORT2 & Veterans & WHAS \\
% \midrule
% Cox & $77.8 \pm 1.9^*$ & $63.4 \pm 1.9$ & \underline{$93.7 \pm 0.0$} & $69.3 \pm 0.3$ & $63.2 \pm 0.1$ & $82.7 \pm 0.0$ & $75.6 \pm 1.0^*$ & $81.7 \pm 0.1$ \\
% CoxBoost & $76.1 \pm 0.9$ & $60.0 \pm 2.6$ & $93.7 \pm 0.0$ & $68.9 \pm 0.6$ & $63.2 \pm 0.2$ & $83.4 \pm 0.0$ & $72.1 \pm 1.9$ & $85.1 \pm 0.1$ \\
% XGBoost & $53.2 \pm 1.5$ & $57.1 \pm 2.7$ & $88.9 \pm 0.1$ & $63.4 \pm 0.9$ & $61.4 \pm 0.5$ & $56.4 \pm 0.7$ & $70.9 \pm 1.5$ & $83.4 \pm 0.3$ \\
% RSF & $\mathbf{80.1 \pm 0.8}$ & $58.5 \pm 1.6$ & $93.7 \pm 0.0$ & $68.6 \pm 0.4$ & $61.6 \pm 0.2$ & \underline{$84.2 \pm 0.1^*$} & \underline{$75.8 \pm 1.0^*$} & \underline{$85.8 \pm 0.1$} \\
% DeepSurv & $70.7 \pm 3.0$ & $64.9 \pm 1.7^*$ & $93.6 \pm 0.0$ & \underline{$69.5 \pm 0.4^*$} & \underline{$63.4 \pm 0.2$} & $82.6 \pm 0.1$ & $\mathbf{76.7 \pm 1.0}$ & $83.7 \pm 0.1$ \\
% DeepHit & \underline{$78.4 \pm 0.9^*$} & $64.6 \pm 2.4^*$ & $93.5 \pm 0.0$ & $65.5 \pm 1.0$ & $61.7 \pm 0.3$ & $82.2 \pm 0.1$ & $72.0 \pm 1.4$ & $82.6 \pm 0.2$ \\
% DSM & $76.8 \pm 1.0$ & \underline{$66.2 \pm 0.8^*$} & $50.0 \pm 0.0$ & $49.9 \pm 0.3$ & $61.3 \pm 0.1$ & $83.6 \pm 0.3$ & $65.4 \pm 0.2$ & $69.7 \pm 0.6$ \\
% \midrule
% FPBoost & $78.1 \pm 0.7$ & $\mathbf{66.6 \pm 3.2}$ & $\mathbf{93.8 \pm 0.0}$ & $\mathbf{69.7 \pm 0.4}$ & $\mathbf{64.0 \pm 0.1}$ & $\mathbf{84.3 \pm 0.4}$ & $74.8 \pm 1.3$ & $\mathbf{89.0 \pm 0.3}$ \\
Model & AIDS & Breast & FLCHAIN & GBSG2 & METABRIC & SUPPORT2 & Veterans & WHAS \\
\midrule
Cox & $77.8 \pm 1.9$ & $63.4 \pm 1.9$ & \underline{$93.7 \pm 0.0$} & $69.3 \pm 0.3$ & $63.2 \pm 0.1$ & $82.7 \pm 0.0$ & $75.6 \pm 1.0$ & $81.7 \pm 0.1$ \\
CoxBoost & $76.1 \pm 0.9$ & $60.0 \pm 2.6$ & $93.7 \pm 0.0$ & $68.9 \pm 0.6$ & $63.2 \pm 0.2$ & $83.4 \pm 0.0$ & $72.1 \pm 1.9$ & $85.1 \pm 0.1$ \\
XGBoost & $53.2 \pm 1.5$ & $57.1 \pm 2.7$ & $88.9 \pm 0.1$ & $63.4 \pm 0.9$ & $61.4 \pm 0.5$ & $56.4 \pm 0.7$ & $70.9 \pm 1.5$ & $83.4 \pm 0.3$ \\
RSF & $\mathbf{80.1 \pm 0.8}$ & $58.5 \pm 1.6$ & $93.7 \pm 0.0$ & $68.6 \pm 0.4$ & $61.6 \pm 0.2$ & \underline{$84.2 \pm 0.1$} & \underline{$75.8 \pm 1.0$} & \underline{$85.8 \pm 0.1$} \\
DeepSurv & $70.7 \pm 3.0$ & $64.9 \pm 1.7$ & $93.6 \pm 0.0$ & \underline{$69.5 \pm 0.4$} & \underline{$63.4 \pm 0.2$} & $82.6 \pm 0.1$ & $\mathbf{76.7 \pm 1.0}$ & $83.7 \pm 0.1$ \\
DeepHit & \underline{$78.4 \pm 0.9$} & $64.6 \pm 2.4$ & $93.5 \pm 0.0$ & $65.5 \pm 1.0$ & $61.7 \pm 0.3$ & $82.2 \pm 0.1$ & $72.0 \pm 1.4$ & $82.6 \pm 0.2$ \\
DSM & $76.8 \pm 1.0$ & \underline{$66.2 \pm 0.8$} & $50.0 \pm 0.0$ & $49.9 \pm 0.3$ & $61.3 \pm 0.1$ & $83.6 \pm 0.3$ & $65.4 \pm 0.2$ & $69.7 \pm 0.6$ \\
\midrule
FPBoost & $78.1 \pm 0.7$ & $\mathbf{66.6 \pm 3.2}$ & $\mathbf{93.8 \pm 0.0}$ & $\mathbf{69.7 \pm 0.4}$ & $\mathbf{64.0 \pm 0.1}$ & $\mathbf{84.3 \pm 0.4}$ & $74.8 \pm 1.3$ & $\mathbf{89.0 \pm 0.3}$ \\
\bottomrule
\end{tabularx}
\end{sc}
\end{small}
\end{center}
\vskip -0.1in
\end{table*}

%% file: tab_ibs.tex
\begin{table*}[t]
\caption{Test IBS ($\downarrow$) and 95\% confidence interval for each model and dataset, averaged across 30 seeded splits (same test, different train and validation sets). To enhance readability, all values are scaled by a factor of 100. Best results are highlighted in \textbf{bold}, while the second best are \underline{underlined}. 
% Results with $^*$ indicate no statistically significant difference between FPBoost and the respective model according to a t-test at the 95\% confidence level.
}
\label{tab:ibs}
\vskip 0.15in
\begin{center}
\begin{small}
\begin{sc}\begin{tabularx}{\linewidth}{@{}lYYYYYYYY@{}}
\toprule
Model & AIDS & Breast & FLCHAIN & GBSG2 & METABRIC & SUPPORT2 & Veterans & WHAS \\
\midrule
% Cox & $\mathbf{5.8 \pm 0.0}$ & $22.2 \pm 1.2$ & $\mathbf{4.6 \pm 0.0}$ & $17.7 \pm 0.1$ & \underline{$19.9 \pm 0.0$} & $13.2 \pm 0.0$ & $13.4 \pm 0.2$ & $14.0 \pm 0.0$ \\
% CoxBoost & $6.2 \pm 0.1$ & $19.8 \pm 0.9$ & $4.7 \pm 0.0$ & \underline{$17.3 \pm 0.2$} & $20.9 \pm 0.1$ & $12.6 \pm 0.0^*$ & $14.7 \pm 0.7$ & $11.9 \pm 0.1$ \\
% XGBoost & $8.5 \pm 0.2$ & $22.8 \pm 0.9$ & $12.1 \pm 0.1$ & $26.9 \pm 0.7$ & $24.3 \pm 0.4$ & $60.6 \pm 0.6$ & $47.4 \pm 1.7$ & $18.7 \pm 0.4$ \\
% RSF & \underline{$5.8 \pm 0.0$} & $18.0 \pm 0.3$ & $4.7 \pm 0.0^*$ & $17.7 \pm 0.2$ & $21.0 \pm 0.1$ & $\mathbf{12.3 \pm 0.0}$ & $\mathbf{12.7 \pm 0.2^*}$ & \underline{$8.5 \pm 0.1^*$} \\
% DeepSurv & $6.2 \pm 0.1$ & $26.4 \pm 1.1$ & $4.7 \pm 0.0$ & $17.5 \pm 0.1$ & $20.4 \pm 0.1$ & $14.6 \pm 0.1$ & $14.4 \pm 0.2$ & $12.2 \pm 0.1$ \\
% DeepHit & $5.8 \pm 0.0$ & $23.9 \pm 1.3$ & $6.3 \pm 0.1$ & $21.4 \pm 0.1$ & $22.7 \pm 0.1$ & $14.7 \pm 0.1$ & $29.4 \pm 0.9$ & $17.2 \pm 0.1$ \\
% DSM & $6.2 \pm 0.0$ & \underline{$17.8 \pm 0.0^*$} & $13.9 \pm 0.0$ & $21.6 \pm 0.0$ & $23.7 \pm 0.0$ & $19.4 \pm 0.3$ & $23.0 \pm 0.1$ & $20.5 \pm 0.0$ \\
% \midrule
% FPBoost & $6.0 \pm 0.0$ & $\mathbf{17.1 \pm 0.9}$ & \underline{$4.7 \pm 0.0$} & $\mathbf{17.1 \pm 0.2}$ & $\mathbf{19.8 \pm 0.0}$ & \underline{$12.5 \pm 0.1$} & \underline{$12.8 \pm 0.3$} & $\mathbf{8.4 \pm 0.3}$ \\
Cox & $\mathbf{5.8 \pm 0.0}$ & $22.2 \pm 1.2$ & $\mathbf{4.6 \pm 0.0}$ & $17.7 \pm 0.1$ & \underline{$19.9 \pm 0.0$} & $13.2 \pm 0.0$ & $13.4 \pm 0.2$ & $14.0 \pm 0.0$ \\
CoxBoost & $6.2 \pm 0.1$ & $19.8 \pm 0.9$ & $4.7 \pm 0.0$ & \underline{$17.3 \pm 0.2$} & $20.9 \pm 0.1$ & $12.6 \pm 0.0$ & $14.7 \pm 0.7$ & $11.9 \pm 0.1$ \\
XGBoost & $8.5 \pm 0.2$ & $22.8 \pm 0.9$ & $12.1 \pm 0.1$ & $26.9 \pm 0.7$ & $24.3 \pm 0.4$ & $60.6 \pm 0.6$ & $47.4 \pm 1.7$ & $18.7 \pm 0.4$ \\
RSF & \underline{$5.8 \pm 0.0$} & $18.0 \pm 0.3$ & $4.7 \pm 0.0$ & $17.7 \pm 0.2$ & $21.0 \pm 0.1$ & $\mathbf{12.3 \pm 0.0}$ & $\mathbf{12.7 \pm 0.2}$ & \underline{$8.5 \pm 0.1$} \\
DeepSurv & $6.2 \pm 0.1$ & $26.4 \pm 1.1$ & $4.7 \pm 0.0$ & $17.5 \pm 0.1$ & $20.4 \pm 0.1$ & $14.6 \pm 0.1$ & $14.4 \pm 0.2$ & $12.2 \pm 0.1$ \\
DeepHit & $5.8 \pm 0.0$ & $23.9 \pm 1.3$ & $6.3 \pm 0.1$ & $21.4 \pm 0.1$ & $22.7 \pm 0.1$ & $14.7 \pm 0.1$ & $29.4 \pm 0.9$ & $17.2 \pm 0.1$ \\
DSM & $6.2 \pm 0.0$ & \underline{$17.8 \pm 0.0$} & $13.9 \pm 0.0$ & $21.6 \pm 0.0$ & $23.7 \pm 0.0$ & $19.4 \pm 0.3$ & $23.0 \pm 0.1$ & $20.5 \pm 0.0$ \\
\midrule
FPBoost & $6.0 \pm 0.0$ & $\mathbf{17.1 \pm 0.9}$ & \underline{$4.7 \pm 0.0$} & $\mathbf{17.1 \pm 0.2}$ & $\mathbf{19.8 \pm 0.0}$ & \underline{$12.5 \pm 0.1$} & \underline{$12.8 \pm 0.3$} & $\mathbf{8.4 \pm 0.3}$ \\
\bottomrule
\end{tabularx}
\end{sc}
\end{small}
\end{center}
\vskip -0.1in
\end{table*}

%% file: tab_hparams.tex
\begin{table*}[t]
\caption{Hyperparameter configuration used in FPBoost across datasets.}
\label{tab:hparams}
\vskip 0.15in
\begin{center}
\begin{small}
\begin{sc}
\begin{tabularx}{\linewidth}{@{}lYcYYYYYYY@{}}
\toprule
Parameter & AIDS & Breast & FLCHAIN & GBSG2 & METABRIC & SUPPORT2 & Veterans & WHAS \\
\midrule
Estimators & 16 & 32 & 32 & 16 & 32 & 16 & 64 & 128 \\
Weibull & 32 & 1 & 64 & 4 & 1 & 4 & 1 & 16 \\
LogLogistic & 4 & 0 & 1 & 8 & 1 & 8 & 4 & 1 \\
Max Depth & 1 & 1 & 3 & 1 & 1 & 3 & 1 & 6 \\
$\lambda$ & 1.0 & 1.0 & 1.0 & 1.0 & 1.0 & 1.0 & 1.0 & 1.0 \\
$\alpha$  & 0.5 & 0.01 & 0.1 & 0.01 & 0.0 & 0.1 & 0.0  & 0.0 \\
$\gamma$  & 0.0 & 0.25 & 0.0 & 0.0  & --  & 0.0 & --   & --  \\
$\phi(\cdot)$ & ReLU & ReLU & ReLU & ReLU & ReLU & ReLU & ReLU & ReLU\\
Initialization & Random & Random & Random & Random & Random & Random & Random & KM \\
Patience & -- & -- & -- & -- & -- & -- & -- & 16 \\

\bottomrule
\end{tabularx}
\end{sc}
\end{small}
\end{center}
\vskip -0.1in
\end{table*}

%% file: tab_ctd.tex
\begin{table*}[t]
\caption{Test C-TD ($\uparrow$) and 95\% confidence interval for each model and dataset, averaged across 30 seeded splits (same test, different train and validation sets). To enhance readability, all values are scaled by a factor of 100. Best results are highlighted in \textbf{bold}, while the second best are \underline{underlined}. 
% Results with $^*$ indicate no statistically significant difference between FPBoost and the respective model according to a t-test at the 95\% confidence level.
}
\label{tab:ctd}
\vskip 0.15in
\begin{center}
\begin{small}
\begin{sc}\begin{tabularx}{\linewidth}{@{}lYYYYYYYY@{}}
\toprule
Model & AIDS & Breast & FLCHAIN & GBSG2 & METABRIC & SUPPORT2 & Veterans & WHAS \\
\midrule
% Cox & $77.7 \pm 1.9^*$ & $63.3 \pm 2.0$ & $94.1 \pm 0.0$ & $66.5 \pm 0.3$ & $\mathbf{66.0 \pm 0.1}$ & $81.4 \pm 0.0$ & $74.2 \pm 1.1^*$ & $81.6 \pm 0.1$ \\
% CoxBoost & $76.3 \pm 0.9$ & $60.3 \pm 2.3$ & \underline{$94.1 \pm 0.0$} & \underline{$67.8 \pm 0.7$} & $62.2 \pm 0.4$ & $82.3 \pm 0.0$ & $70.8 \pm 1.9$ & $85.0 \pm 0.1$ \\
% XGBoost & $52.9 \pm 1.5$ & $57.7 \pm 2.8$ & $89.6 \pm 0.1$ & $61.6 \pm 1.1$ & $58.7 \pm 0.6$ & $56.0 \pm 0.7$ & $69.6 \pm 1.6$ & $84.2 \pm 0.3$ \\
% RSF & $\mathbf{80.7 \pm 0.8}$ & $58.3 \pm 1.6$ & $94.1 \pm 0.0$ & $67.2 \pm 0.5$ & $61.3 \pm 0.4$ & \underline{$82.9 \pm 0.1^*$} & \underline{$75.0 \pm 1.0$} & \underline{$87.3 \pm 0.1$} \\
% DeepSurv & $70.8 \pm 2.9$ & $65.8 \pm 1.5^*$ & $94.0 \pm 0.0$ & $67.4 \pm 0.5$ & $65.4 \pm 0.4^*$ & $81.3 \pm 0.1$ & $\mathbf{75.3 \pm 1.0}$ & $83.7 \pm 0.1$ \\
% DeepHit & $78.3 \pm 0.9^*$ & $65.7 \pm 2.4^*$ & $93.9 \pm 0.0$ & $63.3 \pm 1.0$ & \underline{$65.6 \pm 0.3^*$} & $81.1 \pm 0.1$ & $70.5 \pm 1.4$ & $82.4 \pm 0.2$ \\
% DSM & $77.3 \pm 1.0$ & $\mathbf{68.5 \pm 0.8^*}$ & $50.0 \pm 0.0$ & $51.8 \pm 0.2$ & $61.6 \pm 0.2$ & $82.4 \pm 0.3$ & $64.8 \pm 0.2$ & $70.1 \pm 0.6$ \\
% \midrule
% FPBoost & \underline{$78.5 \pm 0.7$} & \underline{$67.3 \pm 3.2$} & $\mathbf{94.1 \pm 0.0}$ & $\mathbf{69.1 \pm 0.5}$ & $65.6 \pm 0.2$ & $\mathbf{83.0 \pm 0.4}$ & $73.2 \pm 1.4$ & $\mathbf{90.3 \pm 0.3}$ \\
Cox & $77.7 \pm 1.9$ & $63.3 \pm 2.0$ & $94.1 \pm 0.0$ & $66.5 \pm 0.3$ & $\mathbf{66.0 \pm 0.1}$ & $81.4 \pm 0.0$ & $74.2 \pm 1.1$ & $81.6 \pm 0.1$ \\
CoxBoost & $76.3 \pm 0.9$ & $60.3 \pm 2.3$ & \underline{$94.1 \pm 0.0$} & \underline{$67.8 \pm 0.7$} & $62.2 \pm 0.4$ & $82.3 \pm 0.0$ & $70.8 \pm 1.9$ & $85.0 \pm 0.1$ \\
XGBoost & $52.9 \pm 1.5$ & $57.7 \pm 2.8$ & $89.6 \pm 0.1$ & $61.6 \pm 1.1$ & $58.7 \pm 0.6$ & $56.0 \pm 0.7$ & $69.6 \pm 1.6$ & $84.2 \pm 0.3$ \\
RSF & $\mathbf{80.7 \pm 0.8}$ & $58.3 \pm 1.6$ & $94.1 \pm 0.0$ & $67.2 \pm 0.5$ & $61.3 \pm 0.4$ & \underline{$82.9 \pm 0.1$} & \underline{$75.0 \pm 1.0$} & \underline{$87.3 \pm 0.1$} \\
DeepSurv & $70.8 \pm 2.9$ & $65.8 \pm 1.5$ & $94.0 \pm 0.0$ & $67.4 \pm 0.5$ & $65.4 \pm 0.4$ & $81.3 \pm 0.1$ & $\mathbf{75.3 \pm 1.0}$ & $83.7 \pm 0.1$ \\
DeepHit & $78.3 \pm 0.9$ & $65.7 \pm 2.4$ & $93.9 \pm 0.0$ & $63.3 \pm 1.0$ & \underline{$65.6 \pm 0.3$} & $81.1 \pm 0.1$ & $70.5 \pm 1.4$ & $82.4 \pm 0.2$ \\
DSM & $77.3 \pm 1.0$ & $\mathbf{68.5 \pm 0.8}$ & $50.0 \pm 0.0$ & $51.8 \pm 0.2$ & $61.6 \pm 0.2$ & $82.4 \pm 0.3$ & $64.8 \pm 0.2$ & $70.1 \pm 0.6$ \\
\midrule
FPBoost & \underline{$78.5 \pm 0.7$} & \underline{$67.3 \pm 3.2$} & $\mathbf{94.1 \pm 0.0}$ & $\mathbf{69.1 \pm 0.5}$ & $65.6 \pm 0.2$ & $\mathbf{83.0 \pm 0.4}$ & $73.2 \pm 1.4$ & $\mathbf{90.3 \pm 0.3}$ \\
\bottomrule
\end{tabularx}
\end{sc}
\end{small}
\end{center}
\vskip -0.1in
\end{table*}

%% file: tab_auc.tex
\begin{table*}[t]
\caption{Test AUC ($\uparrow$) and 95\% confidence interval for each model and dataset, averaged across 30 seeded splits (same test, different train and validation sets). To enhance readability, all values are scaled by a factor of 100. Best results are highlighted in \textbf{bold}, while the second best are \underline{underlined}. 
% Results with $^*$ indicate no statistically significant difference between FPBoost and the respective model according to a t-test at the 95\% confidence level.
}
\label{tab:auc}
\vskip 0.15in
\begin{center}
\begin{small}
\begin{sc}\begin{tabularx}{\linewidth}{@{}lYYYYYYYY@{}}
\toprule
Model & AIDS & Breast & FLCHAIN & GBSG2 & METABRIC & SUPPORT2 & Veterans & WHAS \\
\midrule
% Cox & $\mathbf{78.9 \pm 2.0^*}$ & \underline{$63.0 \pm 2.1^*$} & $95.4 \pm 0.0$ & \underline{$77.8 \pm 0.3^*$} & \underline{$69.0 \pm 0.1$} & $91.0 \pm 0.0$ & \underline{$86.3 \pm 1.1^*$} & $84.8 \pm 0.1$ \\
% CoxBoost & $76.4 \pm 1.1^*$ & $60.3 \pm 2.6^*$ & $95.5 \pm 0.0$ & $76.8 \pm 0.7$ & $65.3 \pm 0.6$ & $91.8 \pm 0.0^*$ & $81.8 \pm 2.0$ & $88.2 \pm 0.1$ \\
% XGBoost & $55.8 \pm 2.1$ & $58.1 \pm 3.0$ & $91.4 \pm 0.1$ & $67.2 \pm 1.3$ & $64.9 \pm 0.8$ & $56.6 \pm 0.8$ & $80.8 \pm 1.7$ & $87.3 \pm 0.4$ \\
% RSF & $72.8 \pm 1.6$ & $61.2 \pm 2.2^*$ & \underline{$95.7 \pm 0.0$} & $76.6 \pm 0.5$ & $67.9 \pm 0.3$ & $\mathbf{91.9 \pm 0.0}$ & $82.9 \pm 0.8$ & \underline{$92.1 \pm 0.1$} \\
% DeepSurv & $70.9 \pm 3.2$ & $\mathbf{63.8 \pm 2.2^*}$ & $95.5 \pm 0.0$ & $77.7 \pm 0.4^*$ & $68.3 \pm 0.3$ & $90.6 \pm 0.1$ & $\mathbf{87.6 \pm 1.0}$ & $86.7 \pm 0.1$ \\
% DeepHit & \underline{$78.4 \pm 1.1^*$} & $58.9 \pm 2.6^*$ & $95.5 \pm 0.1$ & $65.9 \pm 1.8$ & $67.4 \pm 0.4$ & $38.0 \pm 0.3$ & $48.0 \pm 4.6$ & $72.1 \pm 0.4$ \\
% DSM & $75.4 \pm 1.3$ & $62.8 \pm 0.8^*$ & $50.0 \pm 0.0$ & $48.6 \pm 0.5$ & $66.8 \pm 0.1$ & $90.8 \pm 0.1$ & $77.4 \pm 0.1$ & $71.5 \pm 0.4$ \\
% \midrule
% FPBoost & $77.1 \pm 0.9$ & $62.9 \pm 3.8$ & $\mathbf{95.8 \pm 0.0}$ & $\mathbf{77.9 \pm 0.4}$ & $\mathbf{71.4 \pm 0.3}$ & \underline{$91.8 \pm 0.0$} & $86.1 \pm 1.0$ & $\mathbf{92.8 \pm 0.3}$ \\
Cox & $\mathbf{78.9 \pm 2.0}$ & \underline{$63.0 \pm 2.1$} & $95.4 \pm 0.0$ & \underline{$77.8 \pm 0.3$} & \underline{$69.0 \pm 0.1$} & $91.0 \pm 0.0$ & \underline{$86.3 \pm 1.1$} & $84.8 \pm 0.1$ \\
CoxBoost & $76.4 \pm 1.1$ & $60.3 \pm 2.6$ & $95.5 \pm 0.0$ & $76.8 \pm 0.7$ & $65.3 \pm 0.6$ & $91.8 \pm 0.0$ & $81.8 \pm 2.0$ & $88.2 \pm 0.1$ \\
XGBoost & $55.8 \pm 2.1$ & $58.1 \pm 3.0$ & $91.4 \pm 0.1$ & $67.2 \pm 1.3$ & $64.9 \pm 0.8$ & $56.6 \pm 0.8$ & $80.8 \pm 1.7$ & $87.3 \pm 0.4$ \\
RSF & $72.8 \pm 1.6$ & $61.2 \pm 2.2$ & \underline{$95.7 \pm 0.0$} & $76.6 \pm 0.5$ & $67.9 \pm 0.3$ & $\mathbf{91.9 \pm 0.0}$ & $82.9 \pm 0.8$ & \underline{$92.1 \pm 0.1$} \\
DeepSurv & $70.9 \pm 3.2$ & $\mathbf{63.8 \pm 2.2}$ & $95.5 \pm 0.0$ & $77.7 \pm 0.4$ & $68.3 \pm 0.3$ & $90.6 \pm 0.1$ & $\mathbf{87.6 \pm 1.0}$ & $86.7 \pm 0.1$ \\
DeepHit & \underline{$78.4 \pm 1.1$} & $58.9 \pm 2.6$ & $95.5 \pm 0.1$ & $65.9 \pm 1.8$ & $67.4 \pm 0.4$ & $38.0 \pm 0.3$ & $48.0 \pm 4.6$ & $72.1 \pm 0.4$ \\
DSM & $75.4 \pm 1.3$ & $62.8 \pm 0.8$ & $50.0 \pm 0.0$ & $48.6 \pm 0.5$ & $66.8 \pm 0.1$ & $90.8 \pm 0.1$ & $77.4 \pm 0.1$ & $71.5 \pm 0.4$ \\
\midrule
FPBoost & $77.1 \pm 0.9$ & $62.9 \pm 3.8$ & $\mathbf{95.8 \pm 0.0}$ & $\mathbf{77.9 \pm 0.4}$ & $\mathbf{71.4 \pm 0.3}$ & \underline{$91.8 \pm 0.0$} & $86.1 \pm 1.0$ & $\mathbf{92.8 \pm 0.3}$ \\
\bottomrule
\end{tabularx}
\end{sc}
\end{small}
\end{center}
\vskip -0.1in
\end{table*}

%% file: tab_mean.tex
\begin{table*}[t]
\caption{Average metrics and 95\% confidence interval across all datasets. To enhance readability, all values are scaled by a factor of 100. Best results are highlighted in \textbf{bold}, while the second best are \underline{underlined}.}
\label{tab:avg}
\vskip 0.15in
\begin{center}
\begin{small}
\begin{sc}
\begin{tabularx}{\linewidth}{@{}XYYYY@{}}
\toprule
Model & C-Index & IBS & C-TD & AUC \\
\midrule
Cox & $75.9 \pm 1.3$ & $13.9 \pm 0.8$ & $75.6 \pm 1.3$ & \underline{$80.8 \pm 1.4$} \\
CoxBoost & $75.3 \pm 1.4$ & $13.5 \pm 0.7$ & $74.9 \pm 1.5$ & $79.5 \pm 1.5$ \\
XGBoost & $66.9 \pm 1.6$ & $27.6 \pm 2.1$ & $66.3 \pm 1.7$ & $70.3 \pm 1.8$ \\
RSF & \underline{$76.0 \pm 1.5$} & \underline{$12.6 \pm 0.7$} & \underline{$75.8 \pm 1.5$} & $80.1 \pm 1.5$ \\
DeepSurv & $75.6 \pm 1.3$ & $14.6 \pm 0.9$ & $75.5 \pm 1.3$ & $80.1 \pm 1.5$ \\
DeepHit & $75.1 \pm 1.4$ & $17.7 \pm 1.0$ & $75.1 \pm 1.3$ & $65.5 \pm 2.2$ \\
DSM & $65.4 \pm 1.4$ & $18.3 \pm 0.7$ & $65.8 \pm 1.4$ & $67.9 \pm 1.7$ \\
\midrule
FPBoost & $\mathbf{77.5 \pm 1.4}$ & $\mathbf{12.3 \pm 0.7}$ & $\mathbf{77.6 \pm 1.3}$ & $\mathbf{82.0 \pm 1.5}$ \\
\bottomrule
\end{tabularx}
\end{sc}
\end{small}
\end{center}
\vskip -0.1in
\end{table*}

%% file: tab_models.tex
\begin{table}[t]
\caption{Summary of survival model type (non-parametric, semi-parametric, and fully parametric) and estimation type (linear, tree-based, or neural network-based) for each model included in the experiments.}
\label{tab:models}
\vskip 0.15in
\begin{center}
\begin{small}
\begin{sc}\begin{tabularx}{\linewidth}{@{}XcY@{}}
\toprule
Model & Type & Estimator \\
\midrule
Cox & Semi-Parametric & Linear \\
CoxBoost & Semi-Parametric & Tree \\
XGBoost & Semi-Parametric & Tree \\
RSF & Non-Parametric & Tree \\
DeepSurv & Semi-Parametric & Neural \\
DeepHit & Fully Parametric & Neural \\
DSM & Fully Parametric & Neural \\
\midrule
FPBoost & Fully Parametric & Tree \\
\bottomrule
\end{tabularx}
\end{sc}
\end{small}
\end{center}
\vskip -0.1in
\end{table}

%% file: tab_cid_summary.tex
\begin{table}[t]
\caption{C-Index average improvement ($\uparrow$) of FPBoost with respect to existing models across all datasets.}
\label{tab:cid_summary}
\vskip 0.15in
\begin{center}
\begin{small}
\begin{sc}\begin{tabularx}{\linewidth}{@{}l|ccc|c@{}}
\toprule
C-Index & Linear & Tree & Neural & All \\
\midrule
Non-Parametric & -- & $+1.5$ & -- & $+1.5$ \\
Semi-Parametric & $+1.6$ & $+6.4$ & $+1.9$ & $+4.1$ \\
Fully Parametric & -- & -- & $+7.3$ & $+7.3$ \\
\midrule
All & $+1.6$ & $+4.8$ & $+5.5$ & $+4.6$ \\
\bottomrule
\end{tabularx}
\end{sc}
\end{small}
\end{center}
\vskip -0.1in
\end{table}

%% file: tab_ibs_summary.tex
\begin{table}[t]
\caption{IBS average improvement ($\downarrow$) of FPBoost with respect to existing models across all datasets. XGB is excluded from the evaluation, given its outlier performance.}
\label{tab:ibs_summary}
\vskip 0.15in
\begin{center}
\begin{small}
\begin{sc}\begin{tabularx}{\linewidth}{@{}l|ccc|c@{}}
\toprule
IBS & Linear & Tree & Neural & All \\
\midrule
Non-Parametric & -- & $-0.3$ & -- & $-0.3$ \\
Semi-Parametric & $-1.5$ & $-1.2$ & $-2.3$ & $-1.7$ \\
Fully Parametric & -- & -- & $-5.6$ & $-5.6$ \\
\midrule
All & $-1.5$ & $-0.7$ & $-4.5$ & $-2.8$ \\
\bottomrule
\end{tabularx}
\end{sc}
\end{small}
\end{center}
\vskip -0.1in
\end{table}

%% file: tab_ctd_summary.tex
\begin{table}[t]
\caption{C-TD average improvement ($\uparrow$) of FPBoost with respect to existing models across all datasets.}
\label{tab:ctd_summary}
\vskip 0.15in
\begin{center}
\begin{small}
\begin{sc}\begin{tabularx}{\linewidth}{@{}l|ccc|c@{}}
\toprule
C-TD & Linear & Tree & Neural & All \\
\midrule
Non-Parametric & -- & $+1.8$ & -- & $+1.8$ \\
Semi-Parametric & $+2.0$ & $+7.1$ & $+2.2$ & $+4.6$ \\
Fully Parametric & -- & -- & $+7.2$ & $+7.2$ \\
\midrule
All & $+2.0$ & $+5.3$ & $+5.5$ & $+4.9$ \\
\bottomrule
\end{tabularx}
\end{sc}
\end{small}
\end{center}
\vskip -0.1in
\end{table}

%% file: tab_auc_summary.tex
\begin{table}[t]
\caption{AUC average improvement ($\uparrow$) of FPBoost with respect to existing models across all datasets.}
\label{tab:auc_summary}
\vskip 0.15in
\begin{center}
\begin{small}
\begin{sc}\begin{tabularx}{\linewidth}{@{}l|ccc|c@{}}
\toprule
AUC & Linear & Tree & Neural & All \\
\midrule
Non-Parametric & -- & $+1.8$ & -- & $+1.8$ \\
Semi-Parametric & $+1.2$ & $+7.1$ & $+1.8$ & $+4.3$ \\
Fully Parametric & -- & -- & $+15.3$ & $+15.3$ \\
\midrule
All & $+1.2$ & $+5.3$ & $+10.8$ & $+7.1$ \\
\bottomrule
\end{tabularx}
\end{sc}
\end{small}
\end{center}
\vskip -0.1in
\end{table}